\documentclass{article}

\usepackage{PRIMEarxiv}

\usepackage{microtype}
\usepackage{graphicx}
\usepackage{subfigure}
\usepackage{booktabs} 

\usepackage{hyperref}
\usepackage{adjustbox}
\usepackage{algorithm}
\usepackage{algpseudocode}
\usepackage{balance}
\usepackage{amsmath}
\usepackage{amssymb}
\usepackage{mathtools}
\usepackage{amsthm}
\usepackage{multirow}
\usepackage{lipsum}
\theoremstyle{definition}
\newtheorem{example}{Example}[section]
\usepackage{hyperref}
\usepackage{amsthm}

\makeatletter
\def\thmheadbrackets#1#2#3{%
  \thmname{#1}\thmnumber{\@ifnotempty{#1}{ }\@upn{#2}}%
  \thmnote{ {\the\thm@notefont[#3]}}}
\makeatother

\newtheoremstyle{brakets}
  {}
  {}
  {\itshape}
  {}
  {\bfseries}
  {.}
  { }
  {\thmheadbrackets{#1}{#2}{#3}}

\theoremstyle{brakets}

\usepackage{footnote}
\usepackage{etoolbox}
\makeatletter
\patchcmd{\@makecaption}
  {\scshape}
  {}
  {}
  {}
\makeatother
\usepackage{xcolor}

\usepackage{amsfonts}       
\usepackage{nicefrac}       
\usepackage{cite}

\newtheorem{proposition}{Proposition}

\newtheorem{definition}{Definition}
\usepackage{bm}
\usepackage{amsmath}

\DeclareMathOperator*{\argmin}{arg\,min}
\usepackage{dirtytalk}
\usepackage{graphicx,float} 
\usepackage{amssymb}
\usepackage{algorithm}
\usepackage{amsmath,bm}
\usepackage{makecell}

\usepackage{hyperref}
\usepackage{booktabs}
\usepackage{lipsum}
\usepackage{multicol}
\usepackage{graphicx}
\usepackage[capitalize,noabbrev]{cleveref}

\theoremstyle{plain}
\theoremstyle{definition}
\theoremstyle{remark}

\newif\ifcomments
\commentstrue
\ifcomments
    \usepackage{ulem}
    \def\picomment#1{{$ $\color{green} [PI: #1]}}
    \def\piedit#1#2{{{{\sout{\color{red}#1}}}}{{\color{blue}#2}}}
\else 
    \def\picomment#1{}
    \def\piedit#1#2{}{#2}
\fi

\ifcomments
    \usepackage{ulem}
    \def\saedit#1{{$ $\color{blue} [SA: #1]}}
    
\else 
    \def\saedit#1{}
    
\fi

\title{Trade-off between reconstruction loss and feature alignment for domain generalization}

\author{
  Thuan Nguyen \\
 Department of Computer Science \\ Tufts University \\ Medford, MA 02155 \\
  \texttt{Thuan.Nguyen@tufts.edu} \\
   \And
  Boyang Lyu \\
  Department of Electrical and Computer Engineering \\ Tufts University \\ Medford, MA 02155\\
  \texttt{Boyang.Lyu@tufts.edu} \\
     \And
  Prakash Ishwar \\
  Department of Electrical and Computer Engineering \\ Boston University \\ Boston, MA 02215\\
  \texttt{pi@bu.edu} \\
  \And
    Matthias Scheutz \\
 Department of Computer Science \\ Tufts University \\ Medford, MA 02155 \\
  \texttt{Matthias.Scheutz@tufts.edu} \\
     \And
  Shuchin Aeron \\
  Department of Electrical and Computer Engineering \\ Tufts University \\ Medford, MA 02155\\
  \texttt{Shuchin@ece.tufts.edu}
}

\begin{document}
\maketitle

\begin{abstract}
Domain generalization (DG) is a branch of transfer learning that aims to train the learning models on several seen domains and subsequently apply these pre-trained models to other unseen (unknown but related) domains. To deal with challenging settings in DG where both data and label of the unseen domain are not available at training time, the most common approach is to design the classifiers based on the domain-invariant representation features, \textit{i.e.}, the latent representations that are unchanged and transferable between domains. Contrary to popular belief, we show that designing classifiers based on invariant representation features alone is necessary but insufficient in DG. Our analysis indicates the necessity of imposing a constraint on the reconstruction loss induced by representation functions to preserve most of the relevant information about the label in the latent space. More importantly, we point out the trade-off between minimizing the reconstruction loss and achieving domain alignment in DG. Our theoretical results motivate a new DG framework that jointly optimizes the reconstruction loss and the domain discrepancy. Both theoretical and numerical results are provided to justify our approach. 
\end{abstract}

\section{Introduction}

Domain Generalization (DG) has been widely studied over the past decade \cite{wang2022generalizing},  \cite{zhou2021domain}. Similar to Domain Adaptation (DA), DG aims to design a classifier based on one or several seen (source) domains and then apply these pre-trained classifiers to the unseen (target) domains. While accessing the data from unseen (target) domains is allowed in DA, this is strictly prohibited in DG, leading to a more challenging problem. 

In practice, without any knowledge about unseen domains, one of the most common methods in DG is first to look for domain-invariant features, \textit{i.e.}, the features that are general and transferable between domains, then design pre-trained classifiers based on these features. This approach is known as domain-invariant or domain-alignment representation learning and is widely considered as one of the most promising and efficient approaches in DG  (please see our short survey in Sec. \ref{sec: related work}). Though domain-invariant representation learning is possible to learn invariant features, recent work \cite{johansson2019support} shows that this technique does not account for the information loss caused by non-invertible representation maps which consequently motivates the use of (nearly) invertible representation maps \cite{johansson2019support,lyu2021barycentric}. 

In this paper, based on the theme of domain-invariant representation learning and the information-theoretic point of view, we indicate the necessity of imposing a constraint on the reconstruction loss induced by representation mappings to retain the relevant information about the label in learned features\footnote{Note that a smaller reconstruction loss between input data and its representation implies the representation function is nearly invertible. Thus, our results agree with the work in \cite{johansson2019support}, \cite{lyu2021barycentric} that using nearly invertible representation functions is  necessary for DG.}. In addition, we show that there is a trade-off between minimizing the reconstruction loss and minimizing the discrepancy between domains. 

In this paper, our contributions include:
\vspace{2 pt}
\begin{enumerate}
    \item We derive a lower bound on mutual information between the latent representation and their labels to demonstrate the necessity of imposing a constraint on reconstruction loss  in DG $\bm{\lbrack}$Proposition \ref{prop: 1}$\bm{\rbrack}$.
    \vspace{2 pt}
    
    \item We characterize the trade-off between minimizing the reconstruction loss \textit{vs.} minimizing the discrepancy of joint distributions between domains. In other words, we show that it is impossible to perfectly accomplish these two objectives at the same time $\bm{\lbrack}$Proposition \ref{prop: 2}$\bm{\rbrack}$. 
    \vspace{2 pt}
    
    \item We propose a new DG learning framework that directly accounts for both the reconstruction loss and the discrepancy between domains and demonstrate the efficiency of our proposed framework on several datasets.  
\end{enumerate}

The remainder of this paper is structured as follows. In Section \ref{sec: related work}, we provide a short survey of existing works on DG. In Section \ref{sec: problem setup}, we formally describe the problem formulation and clarify the notations. Section \ref{sec: pre} provides some initial results which support to our main results in Section \ref{sec: main results}. The practical approach is described in Section \ref{sec: practical method}. Finally, we provide the numerical results in Section \ref{sec: numerical results} and conclude in Section \ref{sec: conclusion}.

\section{Related work}
\label{sec: related work}
The most common approach in DG is domain-invariant representation learning which aims to extract domain-invariant features and then design a classifier based on these features \cite{arjovsky2019invariant,lu2021nonlinear,ahuja2021invariance,li2022invariant,du2020learning,ganin2016domain,mahajan2021domain,zhou2020domain,ahuja2020invariant}. Domain-invariant approach, however, requires two key assumptions: (a) the domain-invariant features must exist and be shared between domains, and (b) the domain-invariant features must be strongly correlated with labels \cite{arjovsky2019invariant},  \cite{nagarajan2020understanding,bui2021exploiting,nguyen2022conditional}. In addition, to precisely obtain the domain-invariant features, one usually requires the availability of a sufficiently large number of seen domains at training time \cite{nguyen2022conditional},  \cite{chen2021iterative}. Therefore, if (a) the invariant features are neither existent nor strongly correlated with the label, or (b) the number of observed (seen) domains is not large enough,  domain-invariant methods may fail \cite{rosenfeld2021risks,aubin2021linear,kamath2021does,gulrajani2020search}.

Domain-invariant representation learning can be categorized into two main branches: (a) marginal distribution-invariant methods (covariate-alignment), \textit{i.e.}, learning the features such that their distributions are unchanged according to domains, and (b) conditional distribution-invariant methods (concept-alignment), \textit{i.e.}, learning the features such that the conditional distributions of labels given features are stable from domain to domain. The first branch (covariate-alignment) includes the works in \cite{russo2018source,hu2018duplex, muandet2013domain, li2018domain, ghifary2016scatter, erfani2016robust, jin2020feature, blanchard2021domain, shen2018wasserstein}. Particularly, in \cite{muandet2013domain}, \cite{erfani2016robust}, the authors employ deep neural networks to learn transformations such that the differences between variances of transformed features over seen domains are minimized. Similarly, Li \textit{et al.} \cite{li2018domain} propose a method called Maximum Mean Discrepancy (MMD) that aims to minimize the maximum of the mean discrepancy between marginal distributions in seen domains. Sun and Saenko \cite{sun2016deep} propose a method that not only matches the mean but also synchronizes the covariance of feature distributions over different domains. Shen \textit{et al.} \cite{shen2018wasserstein} minimize the Wasserstein distance between marginal distributions of representation variables from different seen domains in latent space to extract invariant features. In a new approach, Bui \textit{et al.} \cite{bui2021exploiting} desire to learn domain-invariant features (with marginal distributions unchanged according to domains) together with domain-specific features to enhance the generalization performance. The second branch (concept-alignment) includes the works in \cite{arjovsky2019invariant,li2022invariant,ahuja2021invariance,wang2021respecting, zhao2020domain, salaudeen2021exploiting, lu2021nonlinear,ahuja2020invariant}. Particularly, linear/non-linear Invariant Risk Minimization algorithms are proposed in \cite{arjovsky2019invariant,lu2021nonlinear} to find a common optimal linear/non-linear classifier over all observed domains under a key assumption that a common optimal classifier exists if the conditional distributions of the learned features are stable from domain to domain. Li \textit{et al.} \cite{li2022invariant} propose a method to extract domain-invariant features via minimizing the mutual information of label given extracted feature for a given domain. Wang \textit{et al.} \cite{wang2021respecting} minimize the Kullback–Leibler (KL) divergence between conditional distributions in each class to obtain domain-invariant features. There are also several works that aim to simultaneously achieve both covariate-alignment and concept-alignment by minimizing the divergence of joint distributions between domains \cite{nguyen2022joint,yang2020class,guo2021out,li2018domain}. 

It is worth noting that domain-invariant methods may fail under some particular settings, for example, if the labels more strongly correlate with the spurious features than with the true invariant features \cite{ahuja2021invariance},  \cite{nguyen2022conditional}. To prevent the failure of learning models in these particular scenarios, Ahuja \textit{et al.} suggest adding a constraint on the entropy of extracted features to capture the true invariant features \cite{ahuja2021invariance}. In a similar vein, Nguyen \textit{et al.} employ the conditional entropy minimization principle to eliminate the spurious-invariant features \cite{nguyen2022conditional}. 

\section{Problem Formulation}
\label{sec: problem setup}
\subsection{Notations}
Let $\mathcal{X}$, $\mathcal{Z}$, $\mathcal{Y}$ denote the input space, the representation space, and the label space, respectively. For a given family of domains $\mathcal{D}$, suppose that the data from $S$ observed (seen) domains $\mathcal{D}^{(1)},\mathcal{D}^{(2)},\dots,\mathcal{D}^{(S)} \in \mathcal{D}$ is accessible, DG tasks aim to learn a representation function $f: \mathcal{X} \rightarrow \mathcal{Z}$ followed by a classifier $g: \mathcal{Z} \rightarrow \mathcal{Y}$ that generalizes well on an unseen domain $\mathcal{D}^{(u)} \in \mathcal{D}$, $u \neq 1,2,\dots,S$.

Let $X$ denote the input random variable, $Z=f(X)$ denote the extracted feature random variable, $Y$ denote the label random variable in input space, representation space, and label space, respectively.  We use superscription $^{(i)}$ to denote the variables and functions specified on domain $\mathcal{D}^{(i)}$. For example, we use $p^{(i)}(\bm{x})$, $p^{(i)}(\bm{z})$, $p^{(i)}(\bm{x},\bm{z})$ to denote the distribution of input sample $\bm{x}$, the distribution of feature sample $\bm{z} = f(\bm{x})$, and their joint distribution on $\mathcal{D}^{(i)}$, respectively. We use $p^{(i)}(X,Z)$ and $p^{(i)}(Y,Z)$ to denote the joint distribution between input random variable $X$ and its representation random variable $Z$ and the joint distribution between label random variable $Y$ and representation random variable $Z$ in $\mathcal{D}^{(i)}$. Finally, we use $H(A|B)$ and $I(A;B)$ to denote the conditional entropy and  mutual information between two random variables $A$ and $B$, respectively.

\subsection{Problem formulation}
For given $S$ seen domains, a DG task aims to find an optimal representation function $f^*$ by solving the following optimization problem:
\begin{equation}
\label{eq: domain-invariant}
\begin{aligned}
& \min_{\substack{ f: \mathcal{X} \rightarrow \mathcal{Z}}} \quad  \textup{R}^{(u)}(g_f \circ f)
\end{aligned}
\end{equation}where $\textup{R}^{(u)}(g_f \circ f)$ denotes the risk (classification error) introduced by using a representation map $f$ followed by an optimal classifier $g_f$ on unseen domain $D^{(u)}$. Note that for a given $f: \mathcal{X} \rightarrow \mathcal{Z}$, the optimal classifier $g_f: \mathcal{Z} \rightarrow \mathcal{Y}$ completely depends on $f$.

In this paper, under the information-theoretic point of view, we want to solve the following optimization problem:
\begin{equation}
\label{eq: main equation}
\begin{aligned}
& \max_{\substack{ f: \mathcal{X} \rightarrow \mathcal{Z}}
} \quad I^{(u)}(Y;Z)
\end{aligned}
\end{equation}where $I^{(u)}(Y;Z)$ denotes mutual information between the labels and representation features on unseen domain $D^{(u)}$. It is worth noting that a higher mutual information between representation features and its labels likely leads to a higher classification accuracy. Thus, solving (\ref{eq: main equation}) acts as a proxy for minimizing the classification risk on the unseen domain which is the ultimate goal of DG tasks in (\ref{eq: domain-invariant}).

\section{Preliminary}
\label{sec: pre}

This section provides some definitions and preliminary results that support our main results in Sec. \ref{sec: main results}. 

\subsection{Measure of domain discrepancy}
Under DG settings, one only can access to the data from seen domains, therefore, the most common approach is first to learn domain-invariant features and then design a classifier based on these features \cite{arjovsky2019invariant,lu2021nonlinear,ahuja2021invariance,li2022invariant,du2020learning,ganin2016domain,mahajan2021domain,zhou2020domain,ahuja2020invariant}. For a given divergence measure $D(\cdot || \cdot )$ and a seen domain $\mathcal{D}^{(s)}$, previous works on DG usually aim to (a) enforce covariate-alignment,  \textit{i.e.}, minimizing $D(p^{(u)}(Z) || p^{(s)}(Z))$,  or (b) enforce concept-alignment,  \textit{i.e.}, minimizing $D(p^{(u)}(Y|Z) || p^{(s)}(Y|Z))$\footnote{Under DG settings, one cannot directly align the distribution from the unseen domain, therefore, aligning the distribution over all seen domains is usually used as a proxy to achieve this goal.}.  We, however, want to learn a mapping $f$ to minimize the mismatch between joint distributions of seen  and unseen domains\footnote{Even though this condition is restricted, our examples in Appendix \ref{apd: 1} show that achieving covariate-alignment or concept-alignment alone is not enough to guarantee a small classification risk on unseen domain.}. 

\begin{definition}[Domain discrepancy induced by a representation function]
\label{def: 1}
For a representation function $f: \mathcal{X} \rightarrow \mathcal{Z}$, unseen domain $\mathcal{D}^{(u)}$, and seen domain $\mathcal{D}^{(s)}$, the domain-discrepancy between $\mathcal{D}^{(u)}$ and $\mathcal{D}^{(s)}$ induced by $f$ is:
\begin{equation}
\label{eq: epsilon}
K(f) = D(p^{(u)}(Y,Z) || p^{(s)}(Y,Z))
\end{equation}where $D(\cdot || \cdot)$ is a divergence measure that quantifies the mismatch between two distributions. 
\end{definition}

If the mapping $f$ induces $K(f)=0$, the distributions between seen and unseen domains are perfectly aligned. In practice, one usually wants to enforce $K(f) \leq \epsilon$ where $\epsilon$ is a positive number.

\begin{definition}
\label{def: 3}
Let $W(\epsilon)$ denote the maximum discrepancy between mutual information of unseen domain $\mathcal{D}^{(u)}$ and  seen domain $\mathcal{D}^{(s)}$ when the domain discrepancy $K(f)$ does not exceed a positive number $\epsilon$. Formally,
\begin{equation}
    W(\epsilon) = \max_{ \substack{f: \mathcal{X} \rightarrow \mathcal{Z}, \\K(f) \leq \epsilon}} \big| I^{(u)}(Y;Z) - I^{(s)}(Y;Z) \big|
\end{equation}where $I^{(u)}(Y;Z)$ and $I^{(s)}(Y;Z)$ are mutual information between label $Y$ and representation feature $Z$ in unseen domain and seen domain, respectively. 
\end{definition}
If $\epsilon=0$, $K(f)=0$, $I^{(u)}(Y;Z) = I^{(s)}(Y;Z)$, and $W(0)=0$. In addition, it is possible to verify that $W(\epsilon)$ is a monotonically increasing function of $\epsilon$.

\subsection{Measure of reconstruction loss}
Note that by Data Processing Inequality \cite{cover1999elements} and the fact that $Y \rightarrow X \rightarrow Z$ forms a Markov chain, for any representation function $f$:
\begin{equation}
\label{eq: data processing inequality}
I^{(u)}(Y;X)  \geq I^{(u)}(Y;Z)    
\end{equation}where $I^{(u)}(Y;X)$ and $I^{(u)}(Y;Z)$ denote mutual information between label and input and mutual information between label and feature on unseen domain, respectively. The equality happens in (\ref{eq: data processing inequality}) if $f$ is invertible. 

It is worth noting that there may exist non-invertible representation functions that make the equality happens. Indeed, if the label information can be precisely preserved under mapping $f$, \textit{i.e.}, using $Z$ to predict $Y$ is as good as using $X$ to predict $Y$, for example, if $H^{(u)}(Y|X) = H^{(u)}(Y|Z)$, then $I^{(u)}(Y;X) = I^{(u)}(Y;Z)$. However, under DG settings, there is no information about the data, nor the label from unseen domains. Thus, it is impossible to design such non-invertible mappings that perfectly preserve the label information on unseen domain. On the other hand, (nearly) invertible mappings can be constructed regardless of domains which is a possible way to retain the useful information on unseen domains, \textit{i.e.}, making $I^{(u)}(Y;Z)$ close to $I^{(u)}(Y;X)$. In practice, (nearly) invertible mappings can be enforced by minimizing the reconstruction loss between input data and its representation.

\begin{definition}[Reconstruction loss]
\label{def: 4}
For a representation function $f: \mathcal{X} \rightarrow \mathcal{Z}$, and a function $\theta: \mathcal{Z} \rightarrow \mathcal{X}$, the reconstruction loss (on unseen domain) induced by $f$ and $\theta$ is defined by:
\begin{eqnarray}
    R(f,\theta) &=& \int_{\bm{x} \in \mathcal{X}} p^{(u)}(\bm{x}) \, \ell(\bm{x}, \theta (f(\bm{x})) ) \,d \bm{x} \nonumber\\
    &=& \int_{\bm{x} \in \mathcal{X}} \int_{\bm{z} \in \mathcal{Z}} p^{(u)}(\bm{x},\bm{z}) \, \ell(\bm{x}, \theta(\bm{z}) ) \,d \bm{x} \,d \bm{z}
\end{eqnarray}where $\ell(\cdot,\cdot)$ is a distortion function. 
\end{definition}

Usually, $f$ is called \textit{encoder} and $\theta$ is called \textit{decoder}.

\begin{definition}
\label{def: 5}
Let $Q(\gamma)$ denote the maximum mutual information loss (on unseen domain) when the reconstruction loss induced by encoder $f$ and decoder $\theta$ does not exceed a positive number $\gamma$. Formally,
\begin{equation}
Q(\gamma) = \max_{\substack{f: \mathcal{X} \rightarrow \mathcal{Z}, \\ \theta: \mathcal{Z} \rightarrow \mathcal{X}, \\ R(f, \theta) \leq \gamma}}  I^{(u)}(Y;X) - I^{(u)}(Y;Z).      
\end{equation}
\end{definition}
Note that $\gamma=0$ implies $f$ is invertible, leading to $I^{(u)}(Y;X) = I^{(u)}(Y;Z)$. Therefore, $Q(0)=0$. In addition, it is possible to show that $Q(\gamma)$ is a monotonic increasing function of $\gamma$.

\section{Main results}
\label{sec: main results}

Based on definitions and initial results in Sec. \ref{sec: pre}, we point out the necessity of employing the representation functions such that a small reconstruction loss is induced in order to solve the optimization problem in (\ref{eq: main equation}). More interestingly, we show that there is a trade-off between minimizing the reconstruction loss and aligning the joint distributions between domains.

\begin{proposition}[Main result]
\label{prop: 1}  For unseen domain $\mathcal{D}^{(u)}$, seen domain $\mathcal{D}^{(s)}$, and any encoder $f$ and decoder $\theta$:
 \begin{eqnarray*}
I^{(u)}(Y;Z) \geq  \max \Big[ \, I^{(s)}(Y;Z) &-& W \big( K(f) \big); \\
 I^{(u)}(Y;X) &-& Q \big( R(f,\theta) \big)  \Big].   
\end{eqnarray*}
\end{proposition}

\begin{proof}
Please see Appendix \ref{apd: theorem 1}.
\end{proof}

Proposition \ref{prop: 1} points out a possible way to solve the optimization problem proposed in (\ref{eq: main equation}). Particularly, to maximize $I^{(u)}(Y;Z)$, one simultaneously needs to (a) maximize $I^{(s)}(Y;Z) - W(K(f))$, and (b) maximize $I^{(u)}(Y;X) -Q(R(f,\theta))$. Since $I^{(u)}(Y;X)$ is a constant, $W(\cdot)$ and $Q(\cdot)$ are monotonically increasing functions, we need to find an encoder $f$ and a decoder $\theta$ to maximize the mutual information on seen domain $I^{(s)}(Y;Z)$, minimize the domain discrepancy $K(f)$, and minimize the reconstruction loss $R(f,\theta)$, at the same time.

In practice, if the invariant features exist, strongly correlate with the label, and can be precisely learned, there may exist a mapping $f$ such that $I^{(s)}(Y;Z)$ is large and $K(f)$ is small which make the first lower bound $I^{(s)}(Y;Z) - W(K(f))$ is tighter than the second lower bound $I^{(u)}(Y;X) - Q(R(f,\theta))$. However, under some failure cases in literature \cite{rosenfeld2021risks,aubin2021linear,kamath2021does,gulrajani2020search}, for example, if the invariant feature does not exist, or if the invariant feature is not strongly correlated with the label, then $K(f)$ is large and $I^{(s)}(Y;Z)$ is small, and the second lower bound $I^{(u)}(Y;X) -Q(R(f,\theta))$ might be the tighter one. Thus, the traditional approaches that simultaneously target to learn the invariant features (minimizing $K(f)$) and minimize the empirical risk (a proxy for maximizing the mutual information on seen domain $I^{(s)}(Y;Z)$) for optimizing the first lower bound, can fail.

Motivated by Proposition \ref{prop: 1}, one wants to design an encoder $f$ and a decoder $\theta$ to simultaneously minimize both $K(f)$ and $R(f, \theta)$. Unfortunately, we show that it is impossible to optimize $K(f)$ and $R(f, \theta)$ at the same time.

\begin{definition}[Reconstruction-alignment function]
\label{def: 7}
For unseen domain $\mathcal{D}^{(u)}$, seen domain $\mathcal{D}^{(s)}$, and a given decoder $\theta$, the reconstruction-alignment function $T(\gamma)$ is defined by:
\begin{small}
\begin{equation}
\label{eq: definition of T(V)}
\begin{aligned}
T(\gamma) &= \min_{f: \mathcal{X} \rightarrow \mathcal{Z}} K(f) = \min_{f: \mathcal{X} \rightarrow \mathcal{Z}} D(p^{(u)}(Y,Z) || p^{(s)}(Y,Z))\\
\textrm{s.t.} \quad & R(f,\theta)=\int_{\bm{x} \in \mathcal{X}} \int_{\bm{z} \in \mathcal{Z}}  p^{(u)}(\bm{x},\bm{z}) \ell(\bm{x}, \theta(\bm{z}) ) \,d \bm{x} \, d \bm{z} \leq \gamma
\end{aligned}
\end{equation}\end{small}where $\gamma$ is a positive number, $\ell(\cdot, \cdot)$ is a distortion measure, and $D(\cdot || \cdot)$ is a divergence measure.
\end{definition}

The reconstruction-alignment function $T(\gamma)$ is the minimal discrepancy between the joint distributions of the unseen domain $\mathcal{D}^{(u)}$ and seen domain $\mathcal{D}^{(s)}$ that can be obtained if the reconstruction loss (on unseen domain) does not exceed a positive number $\gamma$. We formally characterize the trade-off between minimizing reconstruction loss and achieving domain alignment as below. 
\begin{proposition}[Main result]
\label{prop: 2}
If the divergence measure $D(a || b)$ is convex (in both variables $a$ and $b$), then $T(\gamma)$ defined in (\ref{eq: definition of T(V)}) is:
\begin{enumerate}
    \item monotonically non-increasing, and
    \item convex.
\end{enumerate}
\end{proposition}

\begin{proof}
Please see Appendix \ref{apd: theorem 2}. 
\end{proof}

Sharing some similarities with rate-distortion theory \cite{cover1999elements}, Proposition \ref{prop: 2} characterizes the trade-off between minimizing the domain discrepancy $K(f)$ and minimizing the reconstruction loss $R(f,\theta)$. Since Proposition \ref{prop: 2} holds for any decoder $\theta$, there is no encoder $f$ and decoder $\theta$ that can perfectly minimize the domain discrepancy and the reconstruction loss.

In addition, though the proof of Proposition \ref{prop: 2} is constructed by considering the reconstruction loss on unseen domain, a similar proof holds for seen domains, \textit{i.e.}, there is a universal trade-off between minimizing the domain discrepancy and minimizing the reconstruction loss regardless of domain. Finally, it is worth noting that the assumption about the convexity of the divergence $D(\cdot || \cdot)$ is not too restricted in practice. Indeed, most of the divergence measures, for example, the Kullback-Leibler divergence, are convex \cite{cover1999elements}.

\section{Practical Approach}
\label{sec: practical method}

In this section, motivated by Proposition \ref{prop: 1}, we propose a framework that simultaneously optimizes the domain discrepancy, the reconstruction loss, and the empirical risk on seen domains\footnote{Here, minimizing the empirical risk (on seen domains) is considered as a proxy for maximizing the mutual information between label and representation features. Using mutual information as a direct objective function will be kept in our future work.}. Specifically, we want to minimize the following loss function:
\begin{equation}
\label{eq: practical objective function}
\min_{f, g_f, \theta} \sum_{i=1}^{M} \textup{R}^{(i)}(g_f \circ f) + \alpha L_{\textup{discrepancy}}(f)  + \beta L_{\textup{reconstruction}}(f,\theta),
\end{equation}where the first term is the empirical classification risk over $M$ seen domains, the second term denotes the domain discrepancy, and the third term represents the reconstruction loss.  $\alpha$, and $\beta$ are two positive hyper-parameters that control the trade-off between minimizing these three loss terms. 

Compared to most of the existing works in DG, the main difference of our objective function in (\ref{eq: practical objective function}) comes from the reconstruction loss term which is added in light of Proposition \ref{prop: 1} to preserve the information between the latent representation and its labels on unseen domain. Therefore,  (\ref{eq: practical objective function}) can be practically optimized by adding a decoder (for optimizing the reconstruction loss) into the well-established existing DG models that already handle the empirical risk and the domain discrepancy terms. Practically, we employ the following DG methods: Invariant Risk Minimization (IRM) algorithm  \cite{arjovsky2019invariant}, Maximum Mean Discrepancy (MMD) algorithm \cite{li2018domain}, CORrelation ALignment (CORAL) algorithm  \cite{sun2016deep}, Invariant Risk Minimization-Maximum Mean Discrepancy (IRM-MMD) algorithm \cite{guo2021out}, Information Bottleneck-Invariant Risk Minimization (IB-IRM) algorithm \cite{ahuja2021invariance}, Empirical Risk Minimization (ERM) algorithm \cite{vapnik1999overview}, and Conditional Entropy Minimization (CEM) algorithm \cite{nguyen2022conditional} to minimize the first two terms in (\ref{eq: practical objective function})\footnote{Due to the limited time, we randomly select some algorithms from recent works on DG to add the reconstruction loss term. We encourage the reader to find the details of these algorithms in \cite{arjovsky2019invariant,li2018domain,sun2016deep,guo2021out,ahuja2021invariance,nguyen2022conditional}.}. 

To minimize the reconstruction loss term in (\ref{eq: practical objective function}), we train an encoder $f: \mathcal{X} \rightarrow \mathcal{Z}$ together with a decoder $\theta: \mathcal{Z} \rightarrow \mathcal{X}$ to minimize:
\begin{equation}
\label{eq: recon-loss on seen domain}
  L_{\textup{reconstruction}}(f,\theta)  = \sum_{i=1}^{M}   \int_{\bm{x} \in \mathcal{X}} p^{(i)}(\bm{x}) \ell(\bm{x}, \theta (f(\bm{x})) ) \, d \bm{x}  ,
\end{equation}where the squared-Euclidean distance is selected as the distortion measure, \textit{i.e.}, $\ell(a,b)=(a-b)^2$, and $p^{(i)}(\bm{x})$ denotes the input distribution on domain $\mathcal{D}^{(i)}$, $i=1,2,\dots,M$.

By adding the reconstruction loss term into IRM, MMD, CORAL, IRM-MMD, IB-IRM, ERM, and CEM algorithms, the following new algorithms are constructed: IRM-Rec, MMD-Rec, CORAL-Rec, IRM-MMD-Rec,  IB-IRM-Rec, ERM-Rec, and CEM-Rec, respectively. One of the advantages of employing multiple algorithms for dealing with the first two terms in (\ref{eq: practical objective function}) is that it allows us to evaluate the effectiveness of combining the reconstruction-loss term on a variety of DG methods. Indeed, our numerical results in the next section show that adding the reconstruction loss term leads to improvements in the accuracy of existing DG methods.

\section{Experiments}
\label{sec: numerical results}
\subsection{Datasets}

\textbf{Colored-MNIST (CMNIST) \cite{arjovsky2019invariant}}. The CMNIST dataset is a common DG dataset which was first proposed in \cite{arjovsky2019invariant}. The learning task is to classify a colored digit into two classes: the digit is less than or equal to four or the digit is strictly greater than four. There are three domains in CMNIST, two domains contain 25,000 images each and one domain contains 20,000 images. Here, the color is considered as a spurious feature which is added in a way such that the label is more correlated with the color than with the digit. Due to a strong spurious correlation between colors and labels, any algorithm simply aims to minimize the training error will tend to discover the color rather than the shape of the digit (on seen domains) and therefore fail in the test on unseen domains. More details about the CMNIST dataset can be found in \cite{arjovsky2019invariant}.

\textbf{Covariate-Shift-CMNIST (CS-CMNIST) \cite{ahuja2021empirical}.} The CS-CMNIST dataset is a dataset derived from CMNIST dataset which was first introduced in \cite{ahuja2021empirical}. There are 10 classes in CS-CMNIST dataset where each class corresponds to a digit from zero to nine and each digit is associated with a single color. There are three domains in CS-CMNIST: two training domains and one testing domain, each containing 20,000 images. The color is considered the spurious feature and is added in a way such that the color is more correlated to digits on seen domains than on unseen domains.  More detail about CS-CMNIST can be found in \cite{ahuja2021invariance}, \cite{ahuja2021empirical}. 

\subsection{Compared Methods}

As previously discussed in Sec. \ref{sec: practical method}, by adding the reconstruction loss term, we compare the proposed IRM-Rec, MMD-Rec, CORAL-Rec, IRM-MMD-Rec, IB-IRM-Rec, ERM-Rec, and CEM-Rec algorithms against their original IRM \cite{arjovsky2019invariant}, MMD \cite{li2018domain}, CORAL \cite{sun2016deep}, IRM-MMD \cite{guo2021out}, IB-IRM \cite{ahuja2021invariance}, ERM \cite{vapnik1999overview}, and CEM \cite{nguyen2022conditional} algorithms. 

\subsection{Implementation Details}

For the CMNIST dataset, we utilize the excellent implementation in Domainbed \cite{gulrajani2020search} that employs the MNIST-ConvNet with four convolutional layers as the learning model. 20 trials corresponding to 20 pairs of hyper-parameters $\alpha$ and $\beta$ are randomly selected in $[10^{-1}, 10^4]$. For each trial, the learning rate is randomly picked in $[10^{-4.5},10^{-3.5}]$ while the batch size is randomly selected in $[2^{3},2^{9}]$. 

Since the CS-CMNIST dataset is not available in Domainbed \cite{gulrajani2020search}, we follow the implementation proposed in \cite{ahuja2021invariance} where the learning model is composed of three convolutional layers with feature map dimensions of 256, 128, and 64, respectively. The last layer (linear layer) is used to classify the colored digit back to 10 classes corresponding to 10 digits from zero to nine. We use an SGD optimizer for training with a batch size fixed to 128, the learning rate fixed to $10^{-1}$ and decay every 600 steps with the total number of steps set to 2,000. In contrast to CMNIST, a grid search is performed in CS-CMNIST with $\alpha, \beta \in \{ 0.1, 1, 10, 10^2, 10^3, 10^4\}$.

The training-domain validation set procedure is used for model selection, \textit{i.e.}, selecting the hyper-parameters (the models) that induce the highest validation accuracy on the validation set sampled from seen domains \cite{gulrajani2020search,ahuja2021invariance}.

We repeat the whole experiment three times for CMNIST and five times for CS-CMNIST via selecting different random seeds\footnote{We follow the settings in \cite{gulrajani2020search,ahuja2021invariance}. Particularly,  in \cite{gulrajani2020search}, the experiment is repeated three times while in \cite{ahuja2021invariance}, the experiment is repeated five times.}. For each selected random seed, the whole process of hyper-parameters tuning and model selection is repeated. After the whole process is finished, only the average accuracy and its corresponding standard deviation are reported. Our code can be found at \href{https://github.com/thuan2412/tradeoff_between_domain_alignment_and_reconstruction_loss}{this link}\footnote{\url{https://github.com/thuan2412/tradeoff_between_domain_alignment_and_reconstruction_loss}}.

\subsection{Results and Discussion}
\begin{table}[h]
\centering
\renewcommand{\arraystretch}{1.5}
\scalebox{0.7}{\resizebox{\columnwidth}{!}{\begin{tabular}{ c c c c c}
\hline
Algorithm  & IRM \cite{arjovsky2019invariant} & IB-IRM \cite{ahuja2021invariance}  &  MMD-IRM \cite{guo2021out} & CEM \cite{nguyen2022conditional}\\ 
\hline
Accuracy   & 61.5 $\mp$ 1.5                   &  71.8 $\mp$ 0.7                    & 77.2 $\mp$ 0.9             & 85.7 $\mp$ 0.9  \\ 
\hline
Algorithm  & IRM-Rec                          & IB-IRM-Rec                         & MMD-IRM-Rec                & CEM-Rec     \\ 
\hline
Accuracy   & 71.0 $\mp$ 0.8                   &  75.6 $\mp$ 1.1                    & 79.7 $\mp$ 0.6             & 87.1 $\mp$ 1.3  \\ 
\hline
\end{tabular}}}
\vspace{2 pt}
\caption{Average accuracy (\%) of compared methods on CS-CMNIST dataset.}
\label{table: 1}
\end{table}

\begin{table}[h]
\centering
\renewcommand{\arraystretch}{1.5}
\scalebox{0.7}{\resizebox{\columnwidth}{!}{\begin{tabular}{ c c c c c}
\hline
Algorithm  & IRM \cite{arjovsky2019invariant} & MMD \cite{li2018domain} & ERM \cite{vapnik1999overview} & CORAL \cite{sun2016deep}  \\ 
\hline
Accuracy   &  52.0 $\mp$ 0.1                  & 51.5 $\mp$ 0.2          & 51.5 $\mp$ 0.1               &  51.5 $\mp$ 0.1          \\ 
\hline
Algorithm  & IRM-Rec                          & MMD-Rec                 & ERM-Rec                      & CORAL-Rec               \\ 
\hline
Accuracy   & 51.7 $\mp$ 0.2                   & 51.7 $\mp$ 0.1          & 51.8 $\mp$ 0.1               & 52.0 $\mp$ 0.1      \\ 
\hline
\end{tabular}}}
\vspace{2 pt}
\caption{Average accuracy (\%) of compared methods on CMNIST dataset.}
\label{table: 2}
\end{table}

Table \ref{table: 1} and \ref{table: 2} provide the accuracy of compared methods on CS-CMNIST dataset and  CMNIST dataset, respectively. 

As seen, for the CS-CMNIST dataset, the accuracy of all four tested algorithms has been improved when the reconstruction loss term is added. Particularly, the lowest improvement is 1.4\% observed from CEM algorithm \cite{nguyen2022conditional}, while the largest improvement appears in IRM algorithm \cite{arjovsky2019invariant} with the gain of 9.5\%. We believe that the difference in the improvement between the tested algorithms can be explained by Proposition \ref{prop: 1}. Indeed, it seems like the original CEM algorithm \cite{nguyen2022conditional} already works well on CS-CMNIST, therefore, the first lower bound in Proposition \ref{prop: 1} induced by CEM is pretty tight which leads to a small improvement in accuracy when the reconstruction loss term is added. On the other hand, since IRM algorithm \cite{arjovsky2019invariant} performs poorly on CS-CMNIST, we suspect that the first lower bound in Proposition \ref{prop: 1} induced by IRM is the looser one, leading to a large improvement when the reconstruction loss term is added for optimizing the second bound. 

Compared to CS-CMNIST, CMNIST is a more challenging dataset where all tested algorithms perform poorly. Indeed, because there exists a strong spurious correlation between the colors and the labels of digits in CMNIST, there is no algorithm that works well on CMNIST \cite{gulrajani2020search}. 
However, as observed from Table \ref{table: 2}, three out of four tested algorithms have been improved by adding the reconstruction loss term when tested on CMNIST dataset. Even though the improvement is not substantial with the largest margin being only 0.5\% from the CORAL algorithm, this still demonstrates the usefulness of optimizing the reconstruction loss term in DG. 

It is worth noting that the numerical results for IRM, MMD, CORAL, and ERM on the CMNIST dataset are collected from \cite{gulrajani2020search} while the numerical results for IRM, IB-IRM, and CEM on the CS-CMNIST dataset are collected from \cite{nguyen2022conditional}. Since the source code of MMD-IRM \cite{guo2021out} was not released, we implemented this algorithm by our-self in order to construct the MMD-IRM-Rec algorithm.

Finally, our future work will focus on integrating the reconstruction loss into other state-of-the-art DG algorithms. Of course, using mutual information as a direct objective function instead of empirical risk will also be considered as one of our future works.

\section{Conclusions}
\label{sec: conclusion}
In this paper, we showed that learning domain-invariant representation features is necessary in DG but insufficient to preserve the mutual information between label and representation features on unseen domains. This fact suggests for imposing a constraint on the reconstruction loss between the input and its latent representation in order to preserve most of the relevant information about labels. More importantly, we also point out the trade-off between minimizing the reconstruction loss and achieving domain alignment in DG. In other words, it is impossible to minimize both the reconstruction loss and the domain discrepancy at the same time. Our theoretical results motivate a new practical framework that jointly accounts for both the reconstruction loss and the domain discrepancy to learn the optimal representation features. In practice, our proposed algorithms can provide a comparable or slightly better performance compared to state-of-the-art DG methods. 

\bibliography{example_paper_1}
\bibliographystyle{IEEEtran}

\appendix
\balance
\section{Covariate-alignment or concept-alignment alone is not sufficient to achieve Domain Generalization}
\label{apd: 1}

In this section, we provide two small examples to show that achieving covariate-alignment or  concept-alignment alone is not enough to minimize the classification error on unseen domains. 

\begin{example}[Covariate-alignment alone is not enough]
Suppose that there exits a mapping $f: \mathcal{X} \rightarrow \mathcal{Z}$ such that the marginal distributions of seen and unseen domains in the latent space are perfectly aligned. Particularly, we assume that $\mathcal{Z}=\{0,1\}$, and $p^{(s)}(Z=0) = p^{(u)}(Z=0) = p^{(s)}(Z=1) = p^{(u)}(Z=1) = 0.5$. Next, suppose that there is a mismatch between the conditional distribution between two domains (inadequate concept-alignment), for example, 
\begin{eqnarray*}
& p^{(s)}(Y=0|Z=0)=0.9, \\
& p^{(s)}(Y=1|Z=0)=0.1, \\
& p^{(s)}(Y=0|Z=1)=0.1, \\
& p^{(s)}(Y=1|Z=1)=0.9,
\end{eqnarray*}
and
\begin{eqnarray*}
& p^{(u)}(Y=0|Z=0)=0.1, \\
& p^{(u)}(Y=1|Z=0)=0.9, \\
& p^{(u)}(Y=0|Z=1)=0.9, \\
& p^{(u)}(Y=1|Z=1)=0.1.
\end{eqnarray*}

If one trains a maximum likelihood classifier $g: \mathcal{Z} \rightarrow \mathcal{Y}$ on seen domain, then $g(Z=0)=0$ and $g(Z=1)=1$. The classification error on the seen domain induced by $f$ and $g$ is: 
\begin{eqnarray*}
\textup{R}^{(s)}(g \circ f) &=& p^{(s)}(Z=0) \big[1-p^{(s)}(Y=0|Z=0) \big] \\
&+& p^{(s)}(Z=1) \big[ 1-p^{(s)}(Y=1|Z=1) \big]\\
&=& 0.1.  
\end{eqnarray*}
  
Next, if one applies this pre-trained classifier $g$ to the unseen domain, the classification error is:
\begin{eqnarray*}
\textup{R}^{(u)}(g \circ f) &=& p^{(u)}(Z=0) \big[ 1-p^{(u)}(Y=0|Z=0) \big] \\
&+& p^{(u)}(Z=1) \big[1-p^{(u)}(Y=1|Z=1) \big] \\
&=& 0.9.  
\end{eqnarray*}

Therefore, covariate-alignment alone is not sufficient to guarantee a low classification error on unseen domain.
\end{example}

\begin{example}[Concept-alignment alone is not enough]
Suppose that there exits a mapping $f: \mathcal{X} \rightarrow \mathcal{Z}$ such that the conditional distributions of seen and unseen domains in the latent space are perfectly aligned. Particularly, we assume that $\mathcal{Z}=\{0,1\}$, and
\begin{eqnarray*}
&&p^{(s)}(Y=0|Z=0)=p^{(u)}(Y=0|Z=0)=0.9, \\
&&p^{(s)}(Y=1|Z=0)=p^{(u)}(Y=1|Z=0)=0.1, \\
&&p^{(s)}(Y=0|Z=1)=p^{(u)}(Y=0|Z=1)=0.49, \\
&&p^{(s)}(Y=1|Z=1)=p^{(u)}(Y=1|Z=1)=0.51.
\end{eqnarray*}
Next, suppose that there is a mismatch between the marginal distribution of two domains (inadequate covariate-alignment), for example, $p^{(s)}(Z=0)=0.9$, $p^{(s)}(Z=1)=0.1$ while $p^{(u)}(Z=0)=0.1$, $p^{(u)}(Z=1)=0.9$. 

If one trains a maximum likelihood classifier $g: \mathcal{Z} \rightarrow \mathcal{Y}$ on seen domain, then $g(Z=0)=0$ and $g(Z=1)=1$. The classification error on seen domain induced by $f$ and $g$ is: 
\begin{eqnarray*}
\textup{R}^{(s)}(g \circ f) &=& p^{(s)}(Z=0) \big[1-p^{(s)}(Y=0|Z=0) \big] \\
&+& p^{(s)}(Z=1) \big[ 1-p^{(s)}(Y=1|Z=1) \big] \\
&=& 0.139.  
\end{eqnarray*}
  
Next, if one transfers this pre-trained classifier $g$ to the unseen domain, the classification error is:
\begin{eqnarray*}
\textup{R}^{(u)}(g \circ f) &=& p^{(u)}(Z=0) \big[ 1-p^{(u)}(Y=0|Z=0) \big] \\
&+& p^{(u)}(Z=1) \big[1-p^{(u)}(Y=1|Z=1) \big] \\
&=& 0.451.  
\end{eqnarray*}

Therefore, achieving concept-alignment alone does not ensure a low classification error on unseen domain. 

\end{example}

\section{Proof of Proposition \ref{prop: 1}}
\label{apd: theorem 1}
First, from Definition \ref{def: 3}, for a given $f$:
\begin{eqnarray}
\label{eq: tata0}
W(K(f)) \geq I^{(s)}(Y;Z) - I^{(u)}(Y;Z),  
\end{eqnarray}which is equivalent to:
\begin{eqnarray}
\label{eq: tata}
I^{(u)}(Y;Z)  \geq I^{(s)}(Y;Z) - W(K(f)). 
\end{eqnarray}

Next, from Definition \ref{def: 5}, for given $f$ and $\theta$:
\begin{equation}
Q(R(f,\theta)) \geq I^{(u)}(Y;X) -  I^{(u)}(Y;Z)
\end{equation}which is equivalent to:
\begin{equation}
\label{eq: tata1}
I^{(u)}(Y;Z) \geq I^{(u)}(Y;X) - Q(R(f,\theta)). 
\end{equation}

Combine (\ref{eq: tata}) and (\ref{eq: tata1}), the proof follows.

\section{Proof of Proposition \ref{prop: 2}}
\label{apd: theorem 2}
First, it is worth noting that our proof closely follows to the proof of rate-distortion theory in \cite{cover1999elements}. 

Particularly, consider two positive numbers $\gamma_1$ and $\gamma_2$, and assume that $\gamma_1 \leq \gamma_2$. For a given decoder $\theta$, let $\mathcal{F}_{\gamma_1}$ and $\mathcal{F}_{\gamma_2}$ denote the sets of representation functions $f$ such that $R(f,\theta) \leq \gamma_1$ and $R(f,\theta) \leq \gamma_2$, respectively. From $\gamma_1 \leq \gamma_2$, $\mathcal{F}_{\gamma_1} \subset \mathcal{F}_{\gamma_2}$. Therefore:
\begin{eqnarray}
T(\gamma_1) = \min_{f \in \mathcal{F}_{\gamma_1}} K(f) \geq \min_{f \in \mathcal{F}_{\gamma_2}} K(f) = T(\gamma_2).
\end{eqnarray}

Thus, $T(\gamma)$ is a monotonically non-increasing function of $\gamma$.

Next, let:
\begin{equation}
\begin{aligned}
f_1 \quad &= \argmin_{f: \mathcal{X} \rightarrow \mathcal{Z}} K(f) \quad
\textrm{s.t.} \quad & R(f,\theta) \leq \gamma_1, \label{eq: tata20}
\end{aligned}
\end{equation}
\begin{equation}
\begin{aligned}
f_2 \quad &= \argmin_{f: \mathcal{X} \rightarrow \mathcal{Z}} K(f) \quad
\textrm{s.t.} \quad & R(f,\theta) \leq \gamma_2. \label{eq: tata21}
\end{aligned}
\end{equation}

Let $p^{(u)}_1(Y,Z)$, $p^{(s)}_1(Y,Z)$ be the corresponding joint distributions of $Y$ and $Z$ on unseen and seen domain introduced by $f_1$, and  $p^{(u)}_2(Y,Z)$, $p^{(s)}_2(Y,Z)$ be the corresponding joint distributions of $Y$ and $Z$ on unseen and seen domain introduced by $f_2$, respectively. 

Let $p^{(u)}_1(X,Z)$, $p^{(s)}_1(X,Z)$ be the corresponding joint distributions of $X$ and $Z$ on unseen and seen domain introduced by $f_1$, and  $p^{(u)}_2(X,Z)$, $p^{(s)}_2(X,Z)$ be the corresponding joint distributions of $X$ and $Z$ on unseen and seen domain introduced by $f_2$, respectively.

Note that for any representation function $f$, we have $p^{(u)}(Y,Z)=p^{(u)}(Y|X) \, p^{(u)}(X,Z)$ and $p^{(s)}(Y,Z)=p^{(s)}(Y|X) \, p^{(s)}(X,Z)$ where $p^{(u)}(Y|X)$ and $p^{(s)}(Y|X)$ denote the conditional distribution between label and input data in unseen and seen domain, respectively. Of course, $p^{(u)}(Y|X)$ and $p^{(s)}(Y|X)$ are independent with the mapping $f$. Thus,
\begin{eqnarray}
\label{eq: zzz1}
p^{(u)}_1(Y,Z)=p^{(u)}(Y|X) \, p^{(u)}_1(X,Z),
\end{eqnarray}
\begin{eqnarray}
\label{eq: zzz2}
p^{(u)}_2(Y,Z)=p^{(u)}(Y|X) \, p^{(u)}_2(X,Z),
\end{eqnarray}
and,
\begin{eqnarray}
\label{eq: zzz3}
p^{(s)}_1(Y,Z)=p^{(s)}(Y|X) \, p^{(s)}_1(X,Z),
\end{eqnarray}
\begin{eqnarray}
\label{eq: zzz4}
p^{(s)}_2(Y,Z)=p^{(s)}(Y|X) \, p^{(s)}_2(X,Z).
\end{eqnarray}
 
Next, to prove the convexity of  $T(\gamma)$, we show that:
\begin{equation}
\label{eq: prove convexity of T(V)}
 \lambda T(\gamma_1) + (1-\lambda)T(\gamma_2) \geq T(\lambda \gamma_1 + (1-\lambda) \gamma_2), \end{equation} for any $\lambda \in [0,1]$. 

Now, let:
\begin{eqnarray}
\label{eq: tata29}
p^{(u)}_{\lambda}(X,Z) =\lambda p^{(u)}_1(X,Z) + (1-\lambda) p^{(u)}_2(X,Z),
\end{eqnarray}
\begin{eqnarray}
\label{eq: tata292}
p^{(s)}_{\lambda}(X,Z) =\lambda p^{(s)}_1(X,Z) + (1-\lambda) p^{(s)}_2(X,Z).
\end{eqnarray}

By definition, the left hand side of (\ref{eq: prove convexity of T(V)}) can be rewritten by:
\begin{small}
\begin{eqnarray}
&&   \lambda T(\gamma_1) + (1-\lambda)T(\gamma_2) \nonumber\\
&=& \lambda D(p^{(u)}_1(Y,Z) \, || \, p^{(s)}_1(Y,Z) ) \nonumber\\
&+& (1-\lambda) D(p^{(u)}_2(Y,Z) \, || \, p^{(s)}_2(Y,Z)) \nonumber \\
&=& \lambda D(p^{(u)}(Y|X) p^{(u)}_1(X,Z)  || p^{(s)}(Y|X) p^{(s)}_1(X,Z) ) \label{eq: 10003}\\
&+&   ( 1  - \lambda ) D (p^{(u)} (Y|X)  p^{(u)}_2   (X,Z) || p^{(s)} (Y|X) p^{(s)}_2  (X,Z) ) \label{eq: 10004}\\
&\geq& D(p^{(u)}(Y|X) p^{(u)}_{\lambda}(X,Z) || p^{(s)}(Y|X) p^{(s)}_{\lambda}(X,Z))  \label{eq: tata28}
\end{eqnarray}\end{small}where (\ref{eq: 10003}) and (\ref{eq: 10004}) due to (\ref{eq: zzz1}), (\ref{eq: zzz2}), (\ref{eq: zzz3}), and (\ref{eq: zzz4}); (\ref{eq: tata28}) due to (\ref{eq: tata29}), (\ref{eq: tata292}), and  the convexity of $D(\cdot || \cdot)$.

Let $f_{\lambda}$ is the corresponding function that induces the joint distribution $p^{(u)}_{\lambda}(X,Z)$ and $p^{(s)}_{\lambda}(X,Z)$\footnote{Indeed, we always can construct $f_{\lambda}$ that induces $p^{(u)}_{\lambda}(X,Z)$ and $p^{(s)}_{\lambda}(X,Z)$ by linear interpolating between $f_1$ and $f_2$.}, the reconstruction loss corresponding to $f_{\lambda}$ is: 
\begin{equation}
\label{eq: tata30}
    \gamma_{\lambda} = \int_{\bm{x} \in \mathcal{X}} \int_{\bm{z} \in \mathcal{Z}} p^{(u)}_{\lambda}(\bm{x},\bm{z}) \ell(\bm{x},\theta(\bm{z})) \, d\bm{x} d\bm{z}.
\end{equation}

By Definition \ref{def: 7},
\begin{eqnarray}
 \label{eq: tata42}
D(p^{(u)}(Y|X) \, p^{(u)}_{\lambda}(X,Z) || p^{(s)}(Y|X) \, p^{(s)}_{\lambda}(X,Z)) \geq T(\gamma_{\lambda}). 
\end{eqnarray}

Combine (\ref{eq: tata28}) and (\ref{eq: tata42}):
\begin{equation}
 \label{eq: tata43}
 \lambda T(\gamma_1) + (1-\lambda)T(\gamma_2) \geq  T(\gamma_{\lambda}),  
\end{equation}
or the left hand side of (\ref{eq: prove convexity of T(V)}) is larger or at least equal to $T(\gamma_{\lambda})$. Next, we show that $T(\gamma_{\lambda})$ is at least as large as the right hand side of (\ref{eq: prove convexity of T(V)}). Particularly, we want to show:
\begin{equation}
 \label{eq: tata44}
 T(\gamma_{\lambda}) \geq T(\lambda \gamma_1 + (1-\lambda) \gamma_2).
\end{equation}

Since $T(\gamma)$ is a monotonically non-increasing function, we want to show that:
 \begin{equation}
 \label{eq: tata45}
 \gamma_{\lambda}  \leq  \lambda \gamma_1 + (1-\lambda) \gamma_2.   
 \end{equation}
 
Indeed,
\begin{small}
 \begin{eqnarray}
 & &   \gamma_{\lambda} = \int_{\bm{x}} \int_{\bm{z}} p^{(u)}_{\lambda}(\bm{x},\bm{z}) \ell(\bm{x},\theta(\bm{z})) d\bm{x} d\bm{z} \label{eq: tata47}\\
 &= &   \int_{\bm{x}}  \int_{\bm{z}}   \Big(   \lambda p^{(u)}_1 (\bm{x},  \bm{z})  +  (1 - \lambda) p^{(u)}_2  (\bm{x}, \bm{z})   \Big)   \ell(\bm{x},   \theta(\bm{z})) d\bm{x} d\bm{z} \label{eq: tata48}\\
 &=& \lambda \int_{\bm{x}} \int_{\bm{z}} p^{(u)}_1(\bm{x},\bm{z}) \ell(\bm{x},\theta(\bm{z})) d\bm{x} d\bm{z} \label{eq: tata49-a}\\
 &+&  (1-\lambda) \int_{\bm{x}} \int_{\bm{z}} p^{(u)}_2(\bm{x},\bm{z}) \ell(\bm{x},\theta(\bm{z})) d\bm{x} d\bm{z} \label{eq: tata49}\\
 &\leq& \lambda \gamma_1 + (1-\lambda) \gamma_2 \label{eq: tata50}
 \end{eqnarray} \end{small}with (\ref{eq: tata47}) due to (\ref{eq: tata30}); (\ref{eq: tata48}) due to (\ref{eq: tata29}); (\ref{eq: tata49-a}) and (\ref{eq: tata49}) due to a bit of algebra; (\ref{eq: tata50}) due to (\ref{eq: tata20}) and (\ref{eq: tata21}), respectively. 

From (\ref{eq: tata45}) and (\ref{eq: tata50}), (\ref{eq: tata44}) follows. Finally, from (\ref{eq: tata43}) and (\ref{eq: tata44}), (\ref{eq: prove convexity of T(V)}) follows. The proof is complete.

\end{document}